\documentclass{article}
\usepackage{packages/extra_packages}
\usepackage[preprint]{packages/neurips}
\usepackage{amsthm}
\title{RunningTensor: Generalizing Linear Attention to Higher-Order Recurrent States
}

\author{%
Luca Herranz-Celotti$^1$, 
\textbf{Ermal Rrapaj}$^2$, \textbf{Vincent Guigue}$^1$   \\
$^1$ISIR Sorbonne, France,
$^2$Berkeley Lab, US\\
\texttt{luca.herranzcelotti@sorbonne-universite.fr, ermalrrapaj@lbl.gov}, \\
\texttt{vincent.guigue@agroparistech.fr}
}
\date{}

\begin{document}

\maketitle
\begin{abstract}
Linear attention and state-space models provide linear-time sequence modeling, but their recurrent memory remains a second-order tensor (a matrix), limiting the order of interactions that can be represented in the state. We introduce the RunningTensor, which generalizes this memory to an order-$o$ tensor, updated by a rank-1 outer product and read by contracting against $o-1$ vector queries. Order $2$ recovers linear attention; we study order $3$ as a proof of concept, retaining both recurrent and parallel forms while remaining linear in sequence length $T$ and improving working memory capacity from $\mathcal{O}(W^2)$ to $\mathcal{O}(W^o)$. On synthetic multi-query associative recall, RunningTensor outperforms linear-attention and SSM baselines. After pretraining, it also improves performance on language-understanding and non-synthetic retrieval tasks, suggesting that higher-order recurrent state can provide useful additional memory capacity beyond matrix-valued state.
\end{abstract}

\section{Introduction}

\setlength{\intextsep}{0pt}
\begin{wrapfigure}{r}{0.38\columnwidth}
    \vspace{-2.25\baselineskip}
    \centering
    \includegraphics[width=\linewidth]{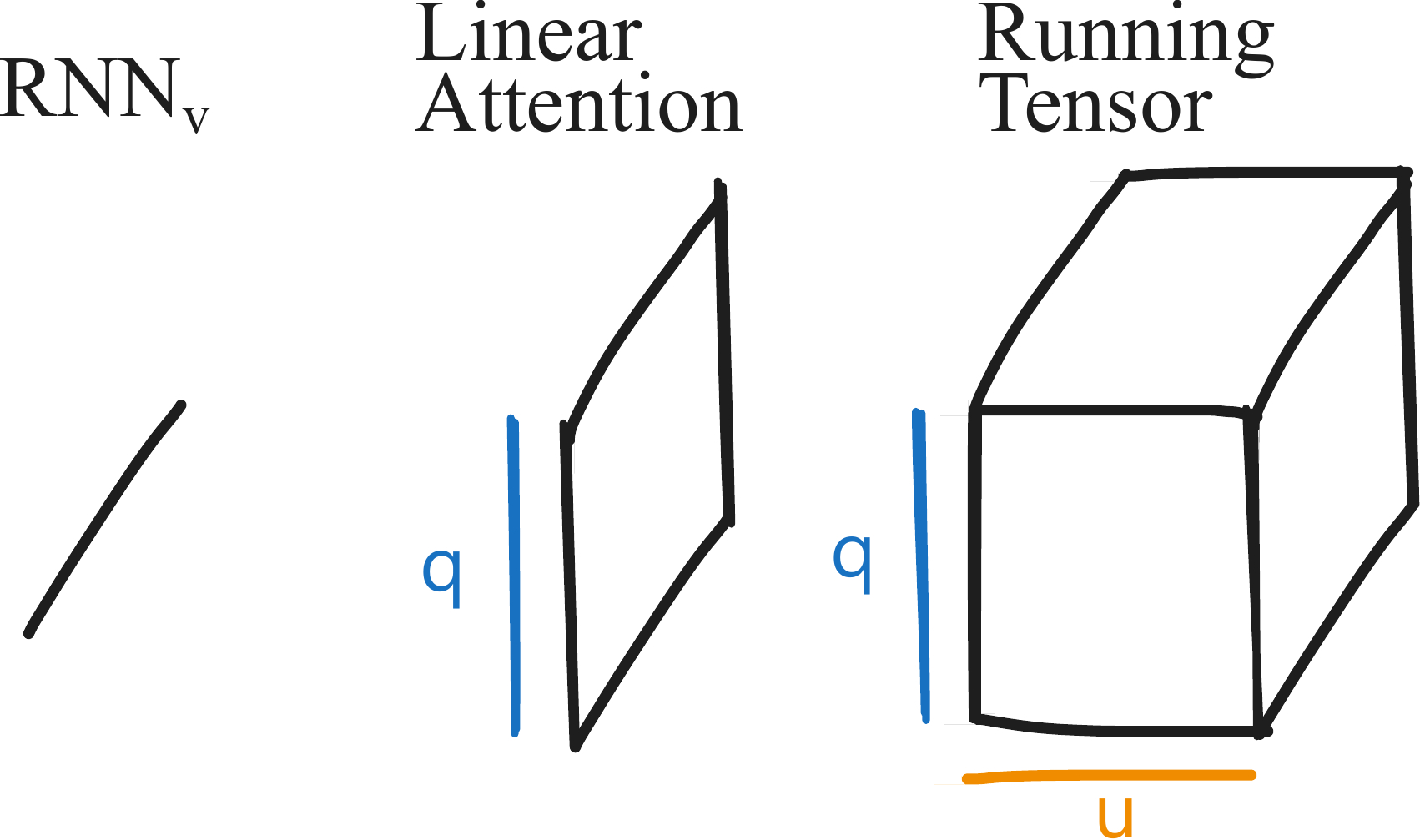}
    \captionsetup{skip=4pt,belowskip=20pt}
    \caption{\textbf{State geometry of linear-time recurrences.} A vector-state RNN ($\text{RNN}_v$) has linear memory capacity, linear attention has quadratic memory capacity given its matrix state, and RunningTensor can increase state size and therefore memory capacity to cubic and higher. One vector query $q$ contraction turns a matrix state into a vector state, but several vector queries $q,u$ are required to turn a tensor state into a vector output.}
    \label{fig:rt-state}
    \vspace{-1.25\baselineskip}
\end{wrapfigure}
Self-attention in Transformers \cite{Vaswani2017,Radford2018,Radford2019} retrieves from the past by materializing pairwise scores between all token positions. The cost is well-documented: that score matrix grows as $\mathcal{O}(T^2)$ in time and memory, which makes long-sequence processing computationally prohibitive.
Linearized attention \cite{Katharopoulos2020} and \emph{state-space models} (SSMs) such as S4 \cite{gu2021ssm}, Mamba \cite{mamba2}, Gated DeltaNet \cite{yang2025gated}, and Longhorn \cite{liu2025longhorn} recast the same retrieval as a linear recurrence with input-dependent gating. The formulation is dual: at inference the model ticks forward one step at a time in $\mathcal{O}(1)$ with respect to sequence length $T$; at training the recurrence unrolls into a matrix multiplication and can be parallelized. This removes the quadratic bottleneck in $T$, but it does not change the geometry of the memory. Linear attention accumulates key-value outer products into a matrix state $\mathbf{S}_t$ and reads through a single contraction $\mathbf{S}_t\mathbf{q}_t$, with a query vector $\mathbf{q}_t$. The summary of the past is therefore bounded by the size and rank of that matrix, which limits memory and expressivity of linearized attention. The question we ask is whether the recurrent state can be lifted beyond this second-order accumulation and single query, without giving up linear cost in $T$.

We propose the \emph{RunningTensor}, a causal sequence model whose state is a tensor of order $o\geq 2$. This work instantiates $o=3$ as a proof of concept: the state is updated by a rank-1 outer product and read by contracting against $o-1$ vector queries, and the construction maps to linear attention when $o=2$. Proposition~\ref{prop:la-special-case} shows that linear attention is contained as a special case, so RunningTensor is strictly more expressive once the extra modes are used. Writing $W$ for the width of each state mode, state size then follows a hierarchy indexed by tensor order: classic vector-state RNNs \cite{Hochreiter1997,gru2014} use $\mathcal{O}(W)$ memory, matrix-state linear attention uses $\mathcal{O}(W^2)$, and an order-$o$ RunningTensor uses $\mathcal{O}(W^o)$, while a full pass remains linear in $T$. The instantiated model adds data-dependent forget gates to reweight past outer products, query-skip gates that let a readout use both tensor axes, one, or none, and ConcOnlyC, which replaces the fused linear projection by a near-parameter-free channel mix.

Our contributions are as follows:
\begin{itemize}
    \item We propose the \emph{RunningTensor}, a linear-time causal sequence model whose state is a tensor of order $o\geq 2$. Order $2$ recovers linear attention; this work instantiates order $3$.
\item On a synthetic multi-query recall benchmark, RunningTensor improves over linear attention and SSM baselines.
\item After pretraining, it also improves on several language-understanding and non-synthetic retrieval downstreams.
\end{itemize}

\section{Methodology: The RunningTensor}

\subsection{Preliminary: Linear Attention}

The linear transformer \cite{Katharopoulos2020} can be written as a linear recurrence when normalization and query/key feature maps are omitted:
\begin{align*}
\mathbf{S}_t &= \mathbf{S}_{t-1} + \mathbf{v}_t \mathbf{k}_t^{\top} \in \mathbb{R}^{d_v \times d_k},
&
\mathbf{o}_t &= \mathbf{S}_t \mathbf{q}_t \in \mathbb{R}^{d_v},
\end{align*}
where $d_k$ and $d_v$ denote the per-head dimensions of the query/key and value vectors, $\mathbf{q}_t, \mathbf{k}_t^{\top}, \mathbf{v}_t$ are usually linear projections of an input representation $\mathbf{x}_t$, $\mathbf{S}_t$ is a matrix-shaped recurrent state and $\mathbf{o}_t$ the output of the unit.
Unrolling the recurrence yields equivalent vector and matrix forms. The sum is causal, which at the element level is a lower-triangular mask $M_{ti}=\mathbf{1}_{i\le t}$:
\begin{align*}
\mathbf{o}_t
&= \sum_{i=1}^{t} (\mathbf{v}_i \mathbf{k}_i^{\top}) \mathbf{q}_t
= \sum_{i=1}^{t} \mathbf{v}_i \bigl(\mathbf{k}_i^{\top} \mathbf{q}_t\bigr) = \sum_{i=1}^{T} M_{ti}\, \mathbf{v}_i \bigl(\mathbf{k}_i^{\top} \mathbf{q}_t\bigr)
\in \mathbb{R}^{d_v}, \\
&\implies
\mathbf{O}
= \bigl(\mathbf{Q} \mathbf{K}^{\top} \odot \mathbf{M}\bigr)\, \mathbf{V}
\in \mathbb{R}^{T \times d_v},
\end{align*}
where $\mathbf{Q}, \mathbf{K} \in \mathbb{R}^{T \times d_k}$, $\mathbf{V} \in \mathbb{R}^{T \times d_v}$, $\mathbf{M} \in \mathbb{R}^{T \times T}$ is the causal mask
and $\odot$ denotes element-wise multiplication. 
This formulation highlights the dual perspective: a recurrent structure enabling linear-time inference, and a matrix form that allows efficient parallel training.

\subsection{The RunningTensor architecture}

\textbf{The RunningTensor.}
To find a tensor analogue, we define a third-order state using the tensor outer product $\otimes$:
\begin{align*}
\boldsymbol{\mathcal{S}}_t
&= \boldsymbol{\mathcal{S}}_{t-1} + \mathbf{v}_t \otimes \mathbf{k}_t \otimes \mathbf{r}_t
\in \mathbb{R}^{d_v \times d_k \times d_r}
\end{align*}
with $\mathbf{v}_t \in \mathbb{R}^{d_v},
\mathbf{k}_t \in \mathbb{R}^{d_k},
\mathbf{r}_t \in \mathbb{R}^{d_r}$.
The same construction extends to order $o>3$: accumulate a rank-1 outer product of $o$ vectors and contract against $o-1$ queries. We restrict to $o=3$ for simplicity.

Contracting modes $2$ and $3$ with queries $\mathbf{q}_t$ and $\mathbf{u}_t$ gives:
\begin{align*}
(o_t)_a
&= \sum_{b=1}^{d_k} \sum_{c=1}^{d_r}
(\mathcal{S}_t)_{abc}\, (q_t)_b\, (u_t)_c,
\qquad a = 1,\ldots,d_v.
\end{align*}
Expanding the recurrence, again with the causal mask at the element level:
\begin{align*}
\mathbf{o}_t
&= \sum_{i=1}^{t}
\mathbf{v}_i\, (\mathbf{k}_i^{\top} \mathbf{q}_t)\, (\mathbf{r}_i^{\top} \mathbf{u}_t) = \sum_{i=1}^{T}
M_{ti}\,
\mathbf{v}_i\, (\mathbf{k}_i^{\top} \mathbf{q}_t)\, (\mathbf{r}_i^{\top} \mathbf{u}_t)
\in \mathbb{R}^{d_v}.\\
&\implies
\mathbf{O}
= \bigl( (\mathbf{Q} \mathbf{K}^{\top}) \odot (\mathbf{U} \mathbf{R}^{\top}) \odot \mathbf{M} \bigr)\, \mathbf{V}
\in \mathbb{R}^{T \times d_v}.
\end{align*}

\begin{proposition}[Linear attention as a special case]
\label{prop:la-special-case}
If $d_r=1$ and $\mathbf{u}_t=\mathbf{r}_t=1$ for all $t$, then the RunningTensor matrix form reduces to causal linear attention.
\end{proposition}
\begin{proof}
With $d_r=1$ and $\mathbf{u}_t=\mathbf{r}_t=1$, one has $\mathbf{r}_i^{\top}\mathbf{u}_t=1$ and $(\mathbf{U}\mathbf{R}^{\top})_{ti}=1$. The expanded readout becomes
$\mathbf{o}_t = \sum_{i=1}^{T} M_{ti}\, \mathbf{v}_i (\mathbf{k}_i^{\top} \mathbf{q}_t)$,
and the matrix form collapses to $\mathbf{O}=(\mathbf{Q}\mathbf{K}^{\top}\odot\mathbf{M})\,\mathbf{V}$.
\end{proof}

\textbf{Forget gates ($\textcolor{lbg}{\alpha}$).}
Without decay, every past outer product enters the state with equal weight. As in several linear attention alternatives, we add data-dependent forgetting so earlier context can be weighted. Per head, meaning several times in parallel, we have
\begin{align*}
\boldsymbol{\mathcal{S}}_t
&= \textcolor{lbg}{\alpha_t}\boldsymbol{\mathcal{S}}_{t-1} + \mathbf{v}_t \otimes \mathbf{k}_t \otimes \mathbf{r}_t.
\end{align*}
Unrolling introduces $\textcolor{lbg}{\gamma_j} =
\prod^j_{i=1}\textcolor{lbg}{\alpha_i}$ and $\textcolor{lbg}{\Gamma_{ts}} = \textcolor{lbg}{\gamma_t}/\textcolor{lbg}{\gamma_s}$. In log space, $\textcolor{lbg}{\Gamma_{ts}} = \exp(\textcolor{lbg}{g_{t}} - \textcolor{lbg}{g_{s}})$ with $\textcolor{lbg}{g_t} = \log \textcolor{lbg}{\gamma_t}$, which is the factor used in the bilinear form.
We parameterize $\textcolor{lbg}{g}$ following Gated Linear Attention (GLA) \cite{yang2024gla}: from a raw per-head logit $\textcolor{lbg}{a_t}$ produced on the fused input path,
\begin{align*}
\textcolor{lbg}{g_t}
&= \frac{1}{\tau}\,\mathrm{log}\sigma(\textcolor{lbg}{a_t})
= \frac{1}{\tau}\log \frac{1}{1+e^{-\textcolor{lbg}{a_t}}},
\qquad \tau = 16,
\end{align*}
so $\textcolor{lbg}{g_t}\in(-\infty,0]$. The per-step multiplier is then $\textcolor{lbg}{\alpha_t} = \exp(\textcolor{lbg}{g_t} - \textcolor{lbg}{g_{t-1}}) = \bigl(\sigma(\textcolor{lbg}{a_t})/\sigma(\textcolor{lbg}{a_{t-1}})\bigr)^{1/\tau}$, which may exceed $1$, in contrast to forget gates constrained to $[0,1]$.
Expanding the sum gives the vector form and, stacking time rows, the parallel matrix form (per head):
\begin{align*}
\mathbf{o}_t
&= \sum_{i=1}^{t}
\mathbf{v}_i
\Bigl(
\textcolor{lbg}{\Gamma_{ti}}\,
\mathbf{k}_i^{\top} \mathbf{q}_t\,
\mathbf{r}_i^{\top} \mathbf{u}_t
\Bigr) = \sum_{i=1}^{T}
M_{ti}\,
\textcolor{lbg}{\Gamma_{ti}}\,
\mathbf{v}_i
\bigl(\mathbf{k}_i^{\top} \mathbf{q}_t\bigr)
\bigl(\mathbf{r}_i^{\top} \mathbf{u}_t\bigr), \\
&\implies
\mathbf{O}
= \Bigl(
(\mathbf{Q} \mathbf{K}^{\top})
\odot
(\mathbf{U} \mathbf{R}^{\top})
\odot \mathbf{M}
\odot \textcolor{lbg}{\boldsymbol{\Gamma}
}\Bigr)\, \mathbf{V}.
\end{align*}

\textbf{Query-skip gates ($\textcolor{xic}{\xi_i}$).}
The bilinear score $(\mathbf{k}_s^{\top}\mathbf{q}_t)(\mathbf{r}_s^{\top}\mathbf{u}_t)$ always queries both axis of the state. We provide the network the ability to decide when to query both axis, one, or none, resulting in a score of the form
\begin{align*}
\Bigl((1-\textcolor{xic}{\xi_{1,t}}) + \textcolor{xic}{\xi_{1,t}}\,\mathbf{k}_s^{\top}\mathbf{q}_t\Bigr)
\Bigl((1-\textcolor{xic}{\xi_{2,t}}) + \textcolor{xic}{\xi_{2,t}}\,\mathbf{r}_s^{\top}\mathbf{u}_t\Bigr),
\qquad \textcolor{xic}{\xi_{1,t}},\textcolor{xic}{\xi_{2,t}}\in(0,1),
\end{align*}
where $\textcolor{xic}{\xi_{i,t}}$ are per-head \emph{query-skip gates} at time $t$, and $i\in\{1,2\}$. Rather than evaluating the mixed score directly, we embed it in the existing recurrence via a homogeneous feature map: dummy slots are appended so that
\begin{align*}
\hat{\mathbf{q}}_t &= [(1-\textcolor{xic}{\xi_{1,t}}),\; \textcolor{xic}{\xi_{1,t}}\mathbf{q}_t],
&
\hat{\mathbf{k}}_s &= [1,\; \mathbf{k}_s],
&
\hat{\mathbf{u}}_t &= [(1-\textcolor{xic}{\xi_{2,t}}),\; \textcolor{xic}{\xi_{2,t}}\mathbf{u}_t],
&
\hat{\mathbf{r}}_s &= [1,\; \mathbf{r}_s],
\end{align*}
and the delta or additive core runs on $(\hat{\mathbf{q}},\hat{\mathbf{k}},\hat{\mathbf{u}},\hat{\mathbf{r}})$. In the full model we additionally $\ell_2$-normalize the expanded queries and keys so pairwise dot products remain bounded, improving delta rule stability.

\textbf{Delta rule $\Delta$ ($\textcolor{betac}{\beta}$).}
For the full model we replace pure accumulation with the \emph{tensor delta rule}, analogous to DeltaNet \cite{irie2021fastweight,yang2024deltanet} and Gated DeltaNet \cite{yang2025gated}, lifted to third order. Let $\langle \boldsymbol{\mathcal{S}}, \mathbf{k}, \mathbf{r}\rangle$ denote mode-wise contraction of $\boldsymbol{\mathcal{S}}\in\mathbb{R}^{d_v\times d_k\times d_r}$ with $\mathbf{k}$ and $\mathbf{r}$. With forget multiplier $\textcolor{lbg}{\alpha_t}$ as above, the recurrent update is
\begin{align*}
\boldsymbol{\mathcal{S}}_t
&= \textcolor{lbg}{\alpha_t}\boldsymbol{\mathcal{S}}_{t-1}
+ \textcolor{betac}{\beta_t}\Bigl(
\mathbf{v}_t - \bigl\langle \textcolor{lbg}{\alpha_t}\boldsymbol{\mathcal{S}}_{t-1}, \hat{\mathbf{k}}_t, \hat{\mathbf{r}}_t \bigr\rangle
\Bigr) \otimes \hat{\mathbf{k}}_t \otimes \hat{\mathbf{r}}_t, \\
\mathbf{o}_t
&= \bigl\langle \boldsymbol{\mathcal{S}}_t, \hat{\mathbf{q}}_t, \hat{\mathbf{u}}_t \bigr\rangle.
\end{align*}
The subtracted prediction $\langle \textcolor{lbg}{\alpha_t} \boldsymbol{\mathcal{S}}_{t-1}, \hat{\mathbf{k}}_t, \hat{\mathbf{r}}_t\rangle$ removes the component of $\mathbf{v}_t$ already represented along $(\hat{\mathbf{k}}_t,\hat{\mathbf{r}}_t)$, so the scalar overwrite gate $\textcolor{betac}{\beta_t}$ controls step size. Training uses a chunk-parallel UT/WY implementation of this recurrence; inference unrolls the same rule step by step.

\textbf{De-parameterize fused linear with ConcOnlyC.}
Queries, keys, values, and all the gates, share one projection rather than separate stacks. Let $\mathbf{x}\in\mathbb{R}^{T \times d_{\mathrm{model}}}$, the fused projection width is
\begin{align*}
d_{\mathrm{fused}}
&= H\bigl(2d_k + 2d_r + 2d_v + 4\bigr),
\end{align*}
for $(\mathbf{q},\mathbf{k},\mathbf{u},\mathbf{r},\mathbf{v},\textcolor{lbg}{a},\textcolor{betac}{b},\textcolor{xic}{z_1},\textcolor{xic}{z_2},\mathbf{g}^{\mathrm{out}})$, $H$ heads, where $\textcolor{lbg}{a}$ is the temporal gate, $\textcolor{betac}{b}$ the delta overwrite gate, $\textcolor{xic}{z_1},\textcolor{xic}{z_2}$ the raw query-skip channels, and $\mathbf{g}^{\mathrm{out}}\in\mathbb{R}^{H d_v}$ the output gate. Gated DeltaNet \cite{yang2025gated} uses a dense linear projection in $\mathbb{R}^{d_{\mathrm{model}}\times d_{\mathrm{fused}}}$.
Instead, we use what we call ConcOnlyC, to expand $\mathbf{x}$  without a dense  matrix. Channels are repeated and truncated (parameter-free tile),
\begin{align*}
\tilde{\mathbf{x}}_{t,:}
&=
\bigl[\mathbf{x}_{t,:},\; \mathbf{x}_{t,:},\; \ldots\bigr]_{1:d_{\mathrm{fused}}},
\end{align*}
then mixed along the expanded width. Viewing $\tilde{\mathbf{x}}\in\mathbb{R}^{T\times d_{\mathrm{fused}}}$ as a length-$d_{\mathrm{fused}}$ signal with $T$ channels (time as channels), we apply a depthwise 1D convolution of odd kernel $K_c=3$ with weights tied across time, so the same filter is applied independently at each $t$:
\begin{align*}
\mathrm{ConcOnlyC}(\mathbf{x})_{t,c}
&= \sum_{k=-(K_c-1)/2}^{(K_c-1)/2}
w_k\, \tilde{\mathbf{x}}_{t,\,c+k}.
\end{align*}
This replaces the fused linear, and $d_{\mathrm{model}}\cdot d_{\mathrm{fused}}$ parameters become $3$. As in Gated DeltaNet \cite{yang2025gated}, a depthwise causal short convolution of kernel $4$ along time then follows on the $(\mathbf{q},\mathbf{k},\mathbf{u},\mathbf{r},\mathbf{v},\textcolor{lbg}{a},\textcolor{betac}{b},\textcolor{xic}{z_1},\textcolor{xic}{z_2},\mathbf{g}^{\mathrm{out}})$ block. After the split, $\mathbf{q},\mathbf{k},\mathbf{u},\mathbf{r},\mathbf{v}$ pass through SiLU, and $\mathbf{q},\mathbf{k},\mathbf{u},\mathbf{r}$ are $\ell_2$-normalized per head before the query-skip feature map expands and re-normalizes them; $\mathbf{v}$ is not. The temporal logits $\textcolor{lbg}{a}$ become $\textcolor{lbg}{g}$ via $\mathrm{log}\sigma/\tau$ as above; the overwrite gate is $\textcolor{betac}{\beta}=\sigma(\textcolor{betac}{b})$ and the query-skip gates are $\textcolor{xic}{\xi_i}=\sigma(\textcolor{xic}{z_i})$. We do not apply $1/\sqrt{d}$ attention scaling since the $\ell_2$-normalization is already playing that role.

\paragraph{The RunningTensor Block.}
As in Gated DeltaNet \cite{yang2025gated}, we apply a multiplicative output gate after the recurrence via FusedRMSNormGated. In their configuration, the gate logits bypass the short convolution so local mixing does not blur them across time. In the full RunningTensor instead, the gate undergoes the same ConcOnlyC channel mix and causal short convolution as the queries and the other gates, yielding a time-local, context-aware modulation of the readout. After the delta (or additive bilinear) core returns $\mathbf{o}\in\mathbb{R}^{T \times H \times d_v}$, FusedRMSNormGated applies $\mathrm{RMSNorm}(\mathbf{o})\odot \mathrm{SiLU}(\mathbf{g}^{\mathrm{out}})$, then a linear map $H d_v\to d_{\mathrm{model}}$ produces the layer output. The decoder residual around the layer is added outside this map.
The RunningTensor sits in a Pre-LN decoder block of the same form as Llama \cite{Touvron2023llama}, Qwen2/3 \cite{yang2024qwen2,yang2025qwen3}, and Mistral \cite{jiang2023mistral}, replacing self-attention.

\subsection{Training and Optimization}

The model is trained end-to-end using next-timestep prediction with a cross-entropy loss.
For a fair comparison, all models are trained under identical conditions with
0.6B parameters on FineWeb-Edu 100BT \cite{penedo2024fineweb}. We
use the AdamW optimizer with a constant learning rate. The final learning rate is selected as the best performing one after one day of training on Qwen only, from the set $\{10^{-3}, 10^{-4}, 10^{-5}\}$, and then shared across all models. We use no weight decay, gradient clipping of 0.01, and one epoch over the full dataset. We also found that stabilizing the variance of activations across layers improves Qwen results. All models employ the Mistral tokenizer \cite{jiang2023mistral} with a vocabulary size of 32{,}000.

\section{Results}

\subsection{Synthetic retrieval (MQAR)}

We evaluate on multi-query associative recall (MQAR) \cite{arora2023zoology} at $d_{\mathrm{model}}=128$ with validation sequence lengths 1024, 2048, and 4096 tokens. Each cell reports the best validation F1 (\%) over 3 seeds and learning rates $\{10^{-2},\,10^{-3},\,10^{-4}\}$. Table~\ref{tab:rt-mqar} compares attention and SSM baselines to RunningTensor (ours).

\begin{table}[h]
\centering
\caption{MQAR validation F1 (\%) at $d_{\mathrm{model}}=128$ (Attention, SSM baselines, and RunningTensor; each cell reports the best validation F1 over 3 seeds and learning rates $\{10^{-2},\,10^{-3},\,10^{-4}\}$; 1024, 2048, and 4096 = MQAR validation sequence length (tokens); ms/step = mean validation forward-pass milliseconds per eval batch step (lower is faster); Params = non-embedding param count per layer; Avg = mean over 1024, 2048, and 4096.)}
\label{tab:rt-mqar}
\begingroup
\makeatletter
\@ifundefined{mqarSTMstrut}{\newsavebox{\mqarSTMstrut}}{}
\setlength{\aboverulesep}{0pt}
\setlength{\belowrulesep}{0pt}
\setlength{\extrarowheight}{0.15ex}
\begin{tabular}{lcc>{\columncolor{gray!20}}c>{\columncolor{gray!20}}c>{\columncolor{gray!20}}cc}
\noalign{\global\setbox\mqarSTMstrut\copy\@arstrutbox}
\toprule
\rule[-0.12ex]{0pt}{0pt}\textbf{Method} & \textbf{Params} & \textbf{ms/step} & \textbf{1024} & \textbf{2048} & \textbf{4096} & \textbf{Avg} \\
\midrule
\rule[-0.12ex]{0pt}{0pt}Attention & 934k & 333.5 & \textbf{100.0} & \textbf{99.9} & \textbf{100.0} & \textbf{100.0} \\
\midrule
\rule[-0.12ex]{0pt}{0pt}Gated $\Delta$ Net & 249k & 261.0 & 78.3 & 57.2 & 48.1 & 61.2 \\
\rule[-0.12ex]{0pt}{0pt}Mamba-2 & 284k & 252.5 & 0.6 & 0.1 & 0.0 & 0.2 \\
\rule[-0.12ex]{0pt}{0pt}GLA & 216k & 267.2 & 0.1 & 1.9 & 0.0 & 0.7 \\
\rule[-0.12ex]{0pt}{0pt}$\Delta$ Net & 249k & 355.7 & 98.5 & 94.5 & \textbf{100.0} & 97.7 \\
\midrule
\rule[-0.12ex]{0pt}{0pt}\textbf{RunningTensor (ours)} & 176k & 311.9 & \textbf{100.0} & \textbf{100.0} & \textbf{100.0} & \textbf{100.0} \\
\bottomrule
\end{tabular}
\makeatother
\endgroup
\end{table}

At matched $d_{\mathrm{model}}$ and fewer parameters than attention, RunningTensor matches perfect MQAR F1 across all evaluated lengths while remaining competitive in forward-pass latency, but faster than attention and faster than the best SSM. Matrix-state and second-order baselines (Gated $\Delta$ Net, Mamba-2, GLA) degrade sharply as context grows; $\Delta$ Net retains strong performance but falls short of the tensor delta variant at 2048 tokens.

\subsection{Component ablation}

Table~\ref{tab:rt-ablation} reports RunningTensor variants at $d_{\mathrm{model}}=128$ under the same protocol as Table~\ref{tab:rt-mqar}. GDN doesn't pass $\mathbf{g}^{\mathrm{out}}$ through the temporal filter, when we do, as in our full architecture, we point it out as $\mathbf{g}^{\mathrm{out}}$@t.

\begin{table}[h]
\centering
\caption{RunningTensor variants on MQAR validation F1 (\%) at $d_{\mathrm{model}}=128$ (same protocol as Table~\ref{tab:rt-mqar}; row labels match internal experiment summaries).}
\label{tab:rt-ablation}
\begingroup
\makeatletter
\@ifundefined{mqarSTMstrut}{\newsavebox{\mqarSTMstrut}}{}
\setlength{\aboverulesep}{0pt}
\setlength{\belowrulesep}{0pt}
\setlength{\extrarowheight}{0.15ex}
\adjustbox{max width=\textwidth,center}{%
\small
\setlength{\tabcolsep}{3.5pt}
\begin{tabular}{lcc>{\columncolor{gray!20}}c>{\columncolor{gray!20}}c>{\columncolor{gray!20}}cc}
\noalign{\global\setbox\mqarSTMstrut\copy\@arstrutbox}
\toprule
\rule[-0.12ex]{0pt}{0pt}\textbf{Method} & \textbf{Params} & \textbf{ms/step} & \textbf{1024} & \textbf{2048} & \textbf{4096} & \textbf{Avg} \\
\midrule
\rule[-0.12ex]{0pt}{0pt}RT w. Tensor $S_t$ + ConcOnlyC + $\xi_i$ + $\Delta$ + $g_{\mathrm{out}}$@t & 176k & 311.9 & \textbf{100.0} & \textbf{100.0} & \textbf{100.0} & \textbf{100.0} \\
\rule[-0.12ex]{0pt}{0pt}RT w. Tensor $S_t$ + ConcOnlyC + $\Delta$ + $g_{\mathrm{out}}$@t & 176k & 322.0 & 99.9 & \textbf{100.0} & \textbf{100.0} & \textbf{100.0} \\
\rule[-0.12ex]{0pt}{0pt}RT w. Tensor $S_t$ + ConcOnlyC + $\xi_i$ + $g_{\mathrm{out}}$@t & 176k & 257.2 & 98.9 & 99.0 & 97.9 & 98.6 \\
\rule[-0.12ex]{0pt}{0pt}RT w. Tensor $S_t$ + ConcOnlyC + $\Delta$ & 175k & 313.7 & 99.7 & \textbf{100.0} & \textbf{100.0} & 99.9 \\
\noalign{\global\setbox\@arstrutbox\hbox{}\global\extrarowheight\z@}
\multicolumn{1}{c}{} & \multicolumn{1}{c}{} & \multicolumn{1}{c}{} & \multicolumn{1}{>{\columncolor{gray!20}}c}{\rule{0pt}{0.64em}} & \multicolumn{1}{>{\columncolor{gray!20}}c}{} & \multicolumn{1}{>{\columncolor{gray!20}}c}{} & \multicolumn{1}{c}{} \\
\noalign{\global\setbox\@arstrutbox\copy\mqarSTMstrut\global\extrarowheight=0.15ex}
\rule[-0.12ex]{0pt}{0pt}RT w. Tensor $S_t$ + ConcOnlyC & 175k & 284.7 & 99.4 & 96.2 & 99.8 & 98.5 \\
\rule[-0.12ex]{0pt}{0pt}RT w. Tensor $S_t$ + ConcOnly & 175k & 275.4 & 81.5 & 75.0 & 81.3 & 79.3 \\
\rule[-0.12ex]{0pt}{0pt}RT w. Tensor $S_t$ + OneLinear & 275k & 254.2 & 97.2 & 98.2 & 54.2 & 83.2 \\
\noalign{\global\setbox\@arstrutbox\hbox{}\global\extrarowheight\z@}
\multicolumn{1}{c}{} & \multicolumn{1}{c}{} & \multicolumn{1}{c}{} & \multicolumn{1}{>{\columncolor{gray!20}}c}{\rule{0pt}{0.64em}} & \multicolumn{1}{>{\columncolor{gray!20}}c}{} & \multicolumn{1}{>{\columncolor{gray!20}}c}{} & \multicolumn{1}{c}{} \\
\noalign{\global\setbox\@arstrutbox\copy\mqarSTMstrut\global\extrarowheight=0.15ex}
\rule[-0.12ex]{0pt}{0pt}RT w. Matrix $S_t$ + ConcOnlyC & 175k & 281.0 & 41.4 & 22.7 & 27.7 & 30.6 \\
\rule[-0.12ex]{0pt}{0pt}RT w. Matrix $S_t$ + ConcOnly & 175k & 268.6 & 44.3 & 4.9 & 13.0 & 20.7 \\
\rule[-0.12ex]{0pt}{0pt}RT w. Matrix $S_t$ + OneLinear & 249k & 260.6 & 51.4 & 10.0 & 19.0 & 26.8 \\
\bottomrule
\end{tabular}
}%
\makeatother
\endgroup
\end{table}

\section{Related Work}

\paragraph{Linear Attention and Fast Weight Mechanisms.}

A prominent direction for improving sequence modeling efficiency replaces quadratic self-attention with linear-time alternatives. Linformer reduces complexity by projecting keys and values into a fixed-dimensional space, while Linear Transformers reformulate attention using kernel-based similarity functions to enable associative computation \cite{wang2020linformer,Katharopoulos2020}. Performer further approximates softmax attention with random feature methods \cite{choromanski2020performer}. Later models enhance stability and expressivity through architectural modifications: RetNet introduces decay and rotational structure, and Gated Linear Attention incorporates learnable gating \cite{sun2023retnet,yang2023gla}. These approaches can be interpreted as instances of fast weight programmers, where weights are dynamically updated via input-dependent outer products, enabling implicit online adaptation during inference \cite{schlag2021linear}. This perspective connects to earlier work on fast weight memory and to recent findings that Transformers can perform implicit optimization while processing sequences \cite{schmidhuber1992fast,schmidhuber1993fast,vonoswald2023mesa}.

\paragraph{State Space Models and Gated Linear RNNs.}

State space models (SSMs) provide an alternative framework for efficient sequence modeling by structuring recurrence in a way that supports parallel computation. Early formulations rely on fixed or structured transition matrices that allow convolutional evaluation \cite{gu2021ssm,li2022convlong}. Subsequent models—including DSS, GSS, S5, H3, and Mamba—improve expressivity through diagonal, low-rank, or selective state transitions \cite{gupta2022dss,mehta2022gss,smith2022s5,fu2022h3,gu2023mamba}. Closely related are linear recurrent architectures such as LRU, HGRN, and RWKV, which share similar computational principles \cite{orvieto2023lru,qin2024hgrn,peng2023rwkv}. A key evolution in this space is the move from data-independent transitions to input-dependent gating mechanisms, inspired by classical gated RNNs but redesigned to remove dependence on previous hidden states, thereby preserving parallelism \cite{gers2000lstm,greff2015lstm,martin2018parallelizing}. This gating paradigm—also termed selective state updates—has become central to modern architectures such as Mamba and its successors.

\paragraph{Delta Rule, Online Learning, and Hybrid Architectures.}

Another unifying perspective frames these models as performing online learning through weight updates. The delta rule, used in DeltaNet, provides greater memory capacity than Hebbian-style updates and leads to flexible identity-plus-low-rank transition structures \cite{gardner1988memory,irie2021fastweight,yang2024deltanet}. These properties support improved reasoning and state tracking but also introduce stability and scalability challenges. Recent work addresses these limitations through parallelization strategies, normalization, and structured approximations \cite{yang2024deltanet,sun2024ttt}. Extensions incorporating nonlinear objectives or richer transition dynamics further enhance expressivity, though often at the cost of reduced parallel efficiency \cite{behrouz2024titans,grazzi2024negative,siems2025householder}. Finally, hybrid architectures that combine attention with linear recurrent or SSM layers—either across or within layers—have emerged as a practical way to balance efficiency and performance \cite{waleffe2024hybrid,hua2022hybrid,ren2024samba}.

\section{Conclusion}

\bibliographystyle{unsrt} 
\bibliography{sections/references}

\end{document}